\documentclass[11pt,letterpaper]{article}
\usepackage[T1]{fontenc}
\usepackage{lmodern}
\usepackage[margin=1in]{geometry}
\usepackage{microtype}
\usepackage{amsmath,amssymb,amsthm}
\usepackage{algorithm,algpseudocode}
\usepackage{booktabs,array}
\usepackage[section]{placeins}
\usepackage[round,authoryear]{natbib}
\usepackage{xcolor}
\definecolor{linkblue}{HTML}{264C75}
\usepackage{hyperref,xurl}
\hypersetup{colorlinks=true,linkcolor=linkblue,citecolor=linkblue,urlcolor=linkblue,
  pdftitle={Revisiting Distributed Sign-Based Optimization with Variance Reduction},
  pdfauthor={Wei Jiang, Zechao Li, Lijun Zhang}}
\newcommand{\E}{\mathbb E}
\newcommand{\R}{\mathbb R}
\newcommand{\Sign}{\operatorname{Sign}}
\newcommand{\norm}[1]{\lVert #1\rVert_2}
\newcommand{\normone}[1]{\lVert #1\rVert_1}
\newcommand{\norminf}[1]{\lVert #1\rVert_\infty}
\newcommand{\del}{\Delta_f}
\newcommand{\ind}[1]{\mathbf1\{#1\}}
\newtheorem{assumption}{Assumption}
\newtheorem{theorem}{Theorem}
\newtheorem{lemma}[theorem]{Lemma}
\newtheorem{corollary}[theorem]{Corollary}
\newtheorem{proposition}[theorem]{Proposition}
\theoremstyle{remark}

\title{Revisiting Distributed Sign-Based  Variance Reduction}
\author{Wei Jiang\textsuperscript{\rm 1,2} \qquad Zechao Li\textsuperscript{\rm 1} \qquad Lijun Zhang\textsuperscript{\rm 2,3}\\
\textsuperscript{\rm 1}\footnotesize School of Computer Science and Engineering, Nanjing University of Science
and Technology, China\\
\textsuperscript{\rm 2}\footnotesize National Key Laboratory for Novel Software Technology,Nanjing University, China\\
\textsuperscript{\rm 3}\footnotesize School of Artificial Intelligence, Nanjing University, China 
 }
\date{}

\begin{document}
\maketitle
\begin{abstract}
Sign-based methods reduce communication costs in distributed environments, but aggregating local signs can introduce bias when data are heterogeneous. As a result, existing sign-based variance reduction methods fail to obtain the optimal convergence rates.
In this paper, we solve this problem and obtain optimal rates for both nonconvex stochastic and finite-sum optimization. We first give a counterexample showing that majority voting can fail to approach stationary points even with exact local gradients. Motivated by this limitation, we propose tracking the global gradient at the server through unbiased compression of recursive gradient increments. As a result, we can obtain the convergence rates of \(O(\sqrt{d/K}+\sqrt d\,(a/(nK))^{1/3})\) for the \(\ell_1\)-norm and \(O(\sqrt{a/K}+\sqrt a/(nK)^{1/3})\) for the \(\ell_2\)-norm. Here, \(K\) is the iteration number, \(n\) is the number of workers, \(d\) is the dimension, and \(a=1+\omega\), with \(\omega\) denoting the compressor’s relative variance. For finite-sum problems with \(M\) components, we combine periodic exact gradient refreshes with compressed component-gradient differences. The resulting total sample complexities are \(O(M+d\sqrt{aM}\epsilon^{-2})\) and \(O(M+a\sqrt M\epsilon^{-2})\) for \(\ell_1\) and \(\ell_2\) gradient norms at most \(\epsilon\), matching the corresponding bounds in centralized settings.
\end{abstract}

\section{Introduction}
\label{sec:introduction}
In this paper, we focus on smooth nonconvex optimization in the distributed setting:
\begin{equation}
\min_{x\in\R^d} f(x)=\frac1n\sum_{j=1}^n f_j(x),
\qquad f_j(x)=\E_{\xi\sim\mathcal D_j}F_j(x;\xi).
\label{eq:problem}
\end{equation}
In this setting, each worker $j\in \{1,2,\cdots,n\}$ accesses its own data distribution $\mathcal D_j$, and these data distributions may differ for different nodes. Such heterogeneous objectives arise when data
are partitioned arbitrarily across machines. We also consider the finite-sum formulation,
in which every worker has a fixed dataset of $m$ components, and the problem can be written in the form:
\begin{equation}
\min_{x\in\R^d} f(x)=\frac1n\sum_{j=1}^n f_j(x),\qquad f_j(x)=\frac1m\sum_{i=1}^m f_{j,i}(x).
\label{eq:fs-problem}
\end{equation}
Here $M=nm$ is the total number of components across the system. Our goal
is to find  stationary points while reducing the information
sent between workers and the parameter server.

Sign-based methods are highly attractive for this purpose since a vector of
signs only takes one bit per coordinate
\citep{bernstein2018signsgd,bernstein2019majority}.
Later studies show that the momentum technique improves the small-batch convergence guarantees
\citep{safaryan2021stochastic,jiang2025momentum}, and variance reduction
further improves the quality of the gradient estimate. In particular,
the previous centralized variance reduction method SSVR attains an $\ell_1$ convergence rate of
$O(\sqrt d\,K^{-1/3})$ under mean-squared smoothness, and SSVR-FS
attains a total gradient complexity of
$O(M+d\sqrt M\,\epsilon^{-2})$ for $\ell_1$ accuracy $\epsilon$
on a centralized dataset of $M$ components
\citep{jiang2024ssvr}. 

These results motivate extending the benefits of
variance reduction to the distributed problem. Such an extension is not automatic. In a majority-vote method, each worker
first converts its local estimate into a sign vector, and the server
takes the majority of those signs. For different local objectives,
this nonlinear aggregation may not align with the gradient of their
average.
SSVR-MV \citep{jiang2024ssvr} provides two options with different server
updates. Option~1 uses deterministic majority voting and has an $\ell_1$
guarantee $O(\sqrt{d/K}+d/\sqrt n)$ with an error floor.
Option~2 uses randomized worker and server signs and obtains an $\ell_2$
rate $O(d^{1/4}K^{-1/4})$. SSVR-MV Option 2 also requires projecting each local
estimate onto a Euclidean ball before worker-side sign quantization, and
its $O(K^{-1/4})$ guarantee leaves a gap to the centralized
variance-reduced $O(K^{-1/3})$ rate.

In this paper, we obtain improved convergence rates with new algorithms. We assume an unbiased relative-variance compressor that introduces
noise proportional to the input magnitude, allowing more general compressors.
We then let the server
retain a global gradient estimate. Workers send compressed recursive
increments, and the server accumulates these messages before choosing
a global update direction. Our recursion compresses an
increment consisting of a scaled gradient and a same-sample
gradient difference, whose second moment is controlled by the
parameter choices. Finally, we use deterministic signs for the
$\ell_1$ criterion and unbiased compression for the $\ell_2$ guarantee.

For the finite-sum problem, we periodically compute
the exact global gradient and use compressed component-gradient
differences for steps between checkpoints. The same component is evaluated
at two successive iterates, so component smoothness controls the update
noise even when local gradients are large or heterogeneous. The resulting total component-gradient complexities are
$O(M+d\sqrt{aM}\epsilon^{-2})$ for $\ell_1$-norm and
$O(M+a\sqrt M\epsilon^{-2})$ for $\ell_2$-norm, where $a=1+\omega$ and $\omega$ is the compressor's relative variance.
These have the same oracle orders as the centralized
SSVR-FS algorithm and Euclidean variance-reduced methods, respectively.

\paragraph{Our contributions.}
We summarize the results and contributions of this paper below.
\begin{itemize}
\item We first analyze the obstruction to local majority voting. We give a three-worker counterexample in Proposition~\ref{prop:majority-floor}, which suffers a nonzero average-gradient floor despite exact local gradients.
It isolates the bias of the final vote and motivates tracking the global
gradient before taking signs.
\item For stochastic problems, we attain
$O(\sqrt{d/K}+\sqrt d(a/(nK))^{1/3})$ in the $\ell_1$ criterion,
removing the previous error floor. Then, we obtain 
$O(\sqrt{a/K}+\sqrt a /(nK)^{1/3})$ for the $\ell_2$-norm.
This improves the horizon dependence of previous methods. 
\item We also obtain the finite-sum bounds for the sign-based distributed setting.
Exact refreshes and compressed component differences remove the need for
a bounded gradient assumption. We obtain total oracle complexities $O(M+d\sqrt{aM}\epsilon^{-2})$ and
$O(M+a\sqrt M\epsilon^{-2})$ for the $\ell_1$ and $\ell_2$ criteria, matching the corresponding centralized rates.
\end{itemize}

\section{Assumptions}
\label{sec:problem}
We now specify the assumptions used.
Write $g_j(x;\xi)=\nabla F_j(x;\xi)$ for a stochastic gradient and
$\del=f(x_1)-f_*$ for the initial objective gap, where $f_*$ is a lower bound of the objective $f$. 

We first give the following assumption on the objective function for the stochastic problem~(\ref{eq:problem}).
\begin{assumption}\label{ass:oracle}
The objective satisfies $f\ge f_*>-\infty$. For every worker and all $x,y$, we have
\begin{align}
\E_\xi g_j(x;\xi)&=\nabla f_j(x),\\
\E_\xi\norm{g_j(x;\xi)}^2&\le H^2,\label{eq:oracle-moment}\\
\E_\xi\norm{g_j(x;\xi)-g_j(y;\xi)}^2
&\le L^2\norm{x-y}^2.\label{eq:ms-smooth}
\end{align}
\end{assumption}
\textbf{Remark:} Here, Jensen's inequality implies that $f_j$ and $f$ are also $L$-smooth.
Besides, inequality~(\ref{eq:oracle-moment}) can be replaced with bounded oracle variance $\sigma^2$ and a true-gradient bound
$\norm{\nabla f_j(x)}\le G_2$, which imply inequality~(\ref{eq:oracle-moment}) with
$H^2=\sigma^2+G_2^2$.

Next, we give the general compression assumption below.
\begin{assumption}[Relative-variance compressor]\label{ass:compressor}
For every input $v$, the output $Q(v)$ satisfies
\begin{equation}
\E[Q(v)\mid v]=v,\qquad
\E[\norm{Q(v)-v}^2\mid v]\le\omega\norm v^2.
\label{eq:relative}
\end{equation}
\end{assumption}
Note that unbiasedness also gives
\begin{equation}
\E[\norm{Q(v)}^2\mid v]\le (1+\omega)\norm v^2.
\label{eq:compressor-second}
\end{equation}
\textbf{Remark:}
The general compression assumption includes the usual scaled stochastic
sign. For $v\ne0$, let $\rho(v)=\norminf v$ and draw independent signs with
$\Pr(S_k=1\mid v)=(1+v_k/\rho(v))/2$. Then
\begin{equation}
Q(v)=\rho(v)S,\qquad Q(0)=0,
\label{eq:scaled-sign}
\end{equation}
is unbiased and satisfies $1+\omega=d$ since $\E[\norm{Q(v)-v}^2\mid v]
=d\norminf v^2-\norm v^2\le(d-1)\norm v^2$.

Finally, we list assumptions for the finite-sum problem \eqref{eq:fs-problem}.
\begin{assumption}\label{ass:finite-sum}
The average objective is lower bounded and each component satisfies
\begin{equation}
\norm{\nabla f_{j,i}(x)-\nabla f_{j,i}(y)}\le L\norm{x-y}
\quad\text{for all }x,y,j,i.
\label{eq:fs-smooth}
\end{equation}
\end{assumption}
The finite-sum results use Assumptions~\ref{ass:compressor}
and~\ref{ass:finite-sum}, requiring no bound on gradients,
oracle variance, or gradient heterogeneity.

\section{The proposed method}
\label{sec:algorithm}
We begin with the earlier SSVR-MV method and explain the bias created by
its majority vote. This motivates a global gradient estimate at the server,
followed by either a sign update for the $\ell_1$ criterion or an unbiased compressed
update for the $\ell_2$ criterion.

\subsection{SSVR-MV and the limitation of its majority voting}
\label{sec:majority-motivation}
\paragraph{The earlier method.}
SSVR-MV, introduced by \citet{jiang2024ssvr}, combines a local
variance-reduced estimator with bidirectional sign communication. Specifically, each worker $j$ first initializes $v_1^j=g_j(x_1;\xi_1^j)$ and, for $t\ge2$, computes
\begin{equation}
v_t^j=g_j(x_t;\xi_t^j)+(1-\beta)
[v_{t-1}^j-g_j(x_{t-1};\xi_t^j)].
\label{eq:local-storm}
\end{equation}
The two stochastic gradients in this update use the same sample. For a fixed
radius $R>0$ and an input satisfying $\norminf v\le R$, define the
random sign vector $S_R(v)$ coordinatewise by
\begin{equation}
\Pr(S_R(v)_k=+1\mid v)=\frac{1+v_k/R}{2},
\qquad \Pr(S_R(v)_k=-1\mid v)=\frac{1-v_k/R}{2}.
\label{eq:original-sign}
\end{equation}
Thus the signs estimate the scaled input without any bias.
In SSVR-MV~(Option 1), worker $j$ sends
$q_t^j=S_R(v_t^j)$, and the server
broadcasts their deterministic majority vote:
\begin{equation}
s_t=\Sign\!\left(\frac1n\sum_{j=1}^n q_t^j\right),
\qquad x_{t+1}=x_t-\eta s_t.
\label{eq:original-majority}
\end{equation}
The setting of their Theorem~3 uses $\norminf{g_j(x;\xi)}\le G$,
$\beta=1/2$, $\eta=\mathcal{O}((dK)^{-1/2})$, and $R=4G$, keeping the inputs inside the
radius. The server forms a new vote
at every iteration.

\paragraph{Where the bias enters.}
The worker signs are unbiased for their scaled inputs, but the majority
operation is nonlinear. This can be seen exactly with three workers.
For any $q_1,q_2,q_3\in\{-1,1\}$,
\begin{equation}
\Sign(q_1+q_2+q_3)
=\frac{q_1+q_2+q_3-q_1q_2q_3}{2}.
\label{eq:majority-polynomial}
\end{equation}
Writing $m_j=\E q_j$,
with all expectations conditional on these inputs, we obtain
\begin{align}
\E\Sign(q_1+q_2+q_3)
&=\frac12\left(\sum_{j=1}^3\E q_j-\E[q_1q_2q_3]\right)\nonumber\\
&=\frac12(m_1+m_2+m_3-m_1m_2m_3).
\label{eq:majority-identity}
\end{align}
If $(m_1,m_2,m_3)=(u,u,-2u)$ with $0<u\le1/2$, the average sign
mean is zero, whereas the expected majority vote is $u^3>0$.
Unbiased worker messages therefore can produce a biased vote.

\begin{proposition}[A gradient floor with exact local estimates]
\label{prop:majority-floor}
There exist three lower-bounded smooth one-dimensional local objectives
with gradients bounded by $G=1$ such that SSVR-MV Option~1,
initialized at a global minimizer and using exact gradients, satisfies
\begin{equation}
\liminf_{K\to\infty}\E|f'(x_\tau)|\ge\frac1{1542}>0
\label{eq:majority-floor-main}
\end{equation}
for $R=4$ and step size
$\eta=\Theta(K^{-1/2})$, where $\tau$ is independent and uniform on $\{1,\ldots,K\}$. 
\end{proposition}
\begin{proof}[Construction and intuition]
Take
\[
f_1(x)=f_2(x)=\tfrac12\log\cosh x+\tfrac14x,
\qquad f_3(x)=\tfrac12\log\cosh x-\tfrac12x.
\]
Their average is $f(x)=\tfrac12\log\cosh x$, minimized at $x=0$.
At this point the worker sign means are $(u,u,-2u)$ with $u=1/16$,
so \eqref{eq:majority-identity} gives a nonzero expected vote despite
$f'(0)=0$. Appendix~\ref{app:majority} proves that this bias produces
the stated gradient floor.
\end{proof}

\subsection{Tracking the global gradient through compressed increments}
\label{sec:server-tracking}
We retain the variance-reduction idea in \eqref{eq:local-storm}, but
change what workers transmit and what the server remembers.
For $t\ge2$, worker $j$ forms the increment
\begin{align}
h_t^j&=g_j(x_t;\xi_t^j)-(1-\beta)g_j(x_{t-1};\xi_t^j)\label{eq:update}\\
&=\beta g_j(x_t;\xi_t^j)+(1-\beta)
[g_j(x_t;\xi_t^j)-g_j(x_{t-1};\xi_t^j)].\nonumber
\end{align}
The worker sends an encoding of $Q(h_t^j)$ and the server then accumulates the increments:
\begin{equation}
z_t=(1-\beta)z_{t-1}+\frac1n\sum_{j=1}^n Q(h_t^j).
\label{eq:server-recursion}
\end{equation}
The difference from \eqref{eq:original-majority} is substantive:
the server retains $z_{t-1}$ and corrects it with numerical
increments before choosing its update direction. In particular, it does not
take a majority vote of independently randomized local signs.

The increment decomposition is central to the analysis. The fresh-gradient
term is multiplied by $\beta$, while mean-squared smoothness controls the
same-sample difference. These two terms give
\begin{equation}
\E_t\norm{h_t^j}^2\le2\beta^2H^2
+2L^2\norm{x_t-x_{t-1}}^2.
\label{eq:update-moment}
\end{equation}
Here $\E_t$ conditions on the history before the current worker samples
and compression calls. Unlike repeatedly compressing a full local
estimate, compressing this increment introduces noise that decreases
with smaller $\beta$ and smaller model movements. The two methods below
control this movement differently. At initialization, every worker compresses $B_0$ independent
sample gradients separately. The whole algorithm is specified in Algorithm~\ref{alg:ucsvr}.

\subsection{The compressed update for the server}
\label{sec:l1}
\paragraph{For $\ell_1$ measure} For the $\ell_1$ criterion, the natural direction is
$s_t=\Sign(z_t)$. With an exact estimate $z_t=\nabla f(x_t)$,
its inner product with the gradient equals $\normone{\nabla f(x_t)}$.
With an approximate estimate, the corresponding inequality is
\begin{equation}
-\langle g,\Sign(z)\rangle\leq-\normone g+2\normone{z-g} \leq-\normone g+2\sqrt d\norm{z-g}.
\label{eq:alignment}
\end{equation}
Thus the deterministic sign update converts a bound on the server's
tracking error into an $\ell_1$ stationarity guarantee. We call this
algorithm \emph{DVR-Sign}: distributed variance reduction with sign updates.

\begin{algorithm}[t]
\caption{DVR-Sign / DVR-Q}
\label{alg:ucsvr}
\label{alg:ucsvr-l2}
\begin{algorithmic}[1]
\Require $x_1,K,B_0,\eta,\beta$, compressor $Q$.
\State Each worker draws $B_0$ samples at $x_1$.
\State Server initializes $z_1=(nB_0)^{-1}\sum_{j,b}Q(g_j(x_1;\xi_{1,b}^j))$.
\For{$t=1,\ldots,K$}
  \If{$t\ge2$}
    \State Each worker draws a sample $\xi_t^j$, forms \eqref{eq:update}, and sends $Q(h_t^j)$.
    \State Server updates $z_t$ by \eqref{eq:server-recursion}.
  \EndIf
  \If{variant is DVR-Sign}
    \State Server broadcasts $s_t=\Sign(z_t)$ to all workers.
  \Else
    \State Server broadcasts $s_t=Q(z_t)$ to all workers.
  \EndIf
  \State All workers and the server set $x_{t+1}=x_t-\eta s_t$.
\EndFor
\State Draw an independent uniform $\tau\in\{1,\ldots,K\}$ and \Return $x_\tau$.
\end{algorithmic}
\end{algorithm}

\paragraph{For $\ell_2$ measure}
An $\ell_1$-norm guarantee already implies an $\ell_2$-norm guarantee through
$\norm g\le\normone g$. However, this transfer retains the
$\sqrt d$ tracking-error coefficient in \eqref{eq:alignment}.
We therefore use the unbiased compressor already defined in
Assumption~\ref{ass:compressor} for the server broadcast as well:
\begin{equation}
s_t=Q(z_t),\qquad x_{t+1}=x_t-\eta s_t.
\label{eq:l2-downlink}
\end{equation}
We call this method \emph{DVR-Q}. Conditional on the history $\mathcal G_t$
containing $z_t$ and all current worker messages, fresh server randomness
gives
\begin{equation}
\E[s_t\mid\mathcal G_t]=z_t,\qquad
\E[\norm{s_t}^2\mid\mathcal G_t]\le (1+\omega)\norm{z_t}^2.
\label{eq:l2-moments}
\end{equation}
With an exact estimate,
the expected gradient inner product is $\norm{\nabla f(x_t)}^2$.
For an approximate estimate, the elementary identity
\begin{equation}
-2\langle g,z\rangle
=-\norm g^2-\norm z^2+\norm{z-g}^2
\label{eq:l2-alignment}
\end{equation}
connects descent to squared Euclidean tracking error. Its negative
$\norm z^2$ term is essential: it absorbs the additional tracking error
caused by the random step length in \eqref{eq:l2-moments}. The two methods thus use different directions for different criteria.
Deterministic signs give the $\ell_1$ inner product in \eqref{eq:alignment}, while
unbiased $Q$ updates give squared $\ell_2$ descent without its explicit
$\sqrt d$ conversion.

\FloatBarrier
\section{Stochastic guarantees}
\label{sec:stochastic-main}
We first analyze DVR-Sign under the $\ell_1$ stationarity criterion and then analyze DVR-Q directly in the squared Euclidean norm. The two
methods share the server recursion, but their model movements and
descent arguments differ. Write $a=1+\omega$ and $c=a/n$.
Denote $e_t=z_t-\nabla f(x_t)$, $E_t=\E\norm{e_t}^2$, and
$\overline E_K=K^{-1}\sum_{t=1}^K E_t$.

\subsection{The \texorpdfstring{$\ell_1$}{l1} guarantee for DVR-Sign}
\label{sec:stochastic-l1}
\begin{theorem}\label{thm:master}
Under Assumptions~\ref{ass:oracle} and~\ref{ass:compressor}, for any $\eta>0$, $\beta\in(0,1]$, and integers
$K,B_0\ge1$, the DVR-Sign variant of Algorithm~\ref{alg:ucsvr} satisfies
\begin{align}
E_1&\le cH^2/B_0,\qquad
E_t\le(1-\beta)E_{t-1}+2c(\beta^2H^2+L^2\eta^2d)\quad(t\ge2),\label{eq:tracking}\\
\overline E_K&\le
\frac{cH^2}{B_0\beta K}
+2c\left(H^2\beta+\frac{L^2\eta^2d}{\beta}\right),\label{eq:average-error}\\
\E\normone{\nabla f(x_\tau)}
&\le\frac{\del}{\eta K}+\frac{L\eta d}{2}
+2\sqrt{d\,\overline E_K}.
\label{eq:master}
\end{align}
\end{theorem}
Balancing these effects gives the following result.

\begin{corollary}\label{cor:rate}
Under the assumptions of Theorem~\ref{thm:master}, choose
\begin{align}
u_1=\frac1{\sqrt K+c^{1/3}K^{2/3}},\qquad
\beta=u_1,\quad \eta=\frac{Hu_1}{L\sqrt d},\quad
B_0=\left\lceil\frac1{u_1^2K}\right\rceil.
\label{eq:rate-parameters0}
\end{align}
Then DVR-Sign satisfies
\begin{align}\label{eq:rate-l1}
\begin{split}
    \E\normone{\nabla f(x_\tau)}
\le & \sqrt d\left[
\left(\frac{L\del}{H}+\frac H2\right)K^{-1/2}
+\left(\frac{L\del}{H}+2\sqrt5H\right)
  \left(\frac cK\right)^{1/3}\right]\\
  = & O\!\left(\sqrt{\frac dK}
+\sqrt d\left(\frac{1+\omega}{nK}\right)^{1/3}\right).
\end{split}
\end{align}
\end{corollary}
\textbf{Remark:} To achieve $\E\normone{\nabla f(x_\tau)} \leq \epsilon$, the required number of model updates is 
\begin{equation*}
K=O\!\left(1+\frac d{\epsilon^2}
              +\frac{ad^{3/2}}{n\epsilon^3}\right).
\end{equation*}

\subsection{The \texorpdfstring{$\ell_2$}{l2} guarantee for DVR-Q}
\label{sec:stochastic-l2}
The unbiased compressor $Q(z_t)$ has conditional mean $z_t$, and its
conditional second moment is at most $(1+\omega)\norm{z_t}^2$. The descent analysis therefore
controls the \emph{squared} gradient norm and retains a negative
tracker-norm term. This term absorbs the tracking error caused by moving
the model, which is the key difference from the sign analysis.

\begin{theorem}\label{thm:master-l2}
Under Assumptions~\ref{ass:oracle} and~\ref{ass:compressor}, for any $\eta>0$, $\beta\in(0,1]$, and integers
$K,B_0\ge1$,  the DVR-Q variant of Algorithm~\ref{alg:ucsvr} satisfies
\begin{align}
&\E\norm{\nabla f(x_\tau)}^2
+\left(1-La\eta-\frac{2caL^2\eta^2}{\beta}\right)
   \frac1K\sum_{t=1}^K\E\norm{z_t}^2\nonumber\\
&\hspace{18mm}\le\frac{2\del}{\eta K}
+\frac{cH^2}{B_0\beta K}+2c\beta H^2.
\label{eq:master-l2}
\end{align}
In particular, if $La\eta\le1/2$ and
$\beta\ge4caL^2\eta^2$, the tracker-norm term on the left is
nonnegative and can be dropped.
\end{theorem}

\begin{corollary}\label{cor:rate-l2}
Under the assumptions of Theorem~\ref{thm:master-l2}, choose
\begin{align}
r=\left(\frac K{n^2}\right)^{1/3},\qquad
\eta=\frac1{2La(1+r)},\quad
\beta=\frac1{n(1+r)^2},\quad B_0=\lceil1+r\rceil.
\label{eq:rate-parameters-l20}
\end{align}
Then DVR-Q satisfies
\begin{align}
\begin{split}
\E\norm{\nabla f(x_\tau)}
&\le\sqrt{4L\del+H^2}\sqrt{\frac aK}
  +\sqrt{4L\del+3H^2}\frac{\sqrt a}{(nK)^{1/3}}\\
&=O\!\left(\sqrt{\frac{1+\omega}{K}}
  +\frac{\sqrt{1+\omega}}{(nK)^{1/3}}\right).
\end{split}
\label{eq:rate-l2}
\end{align}
\end{corollary}
\textbf{Remark:} To achieve $\E\norm{\nabla f(x_\tau)} \leq \epsilon$, the required number of model updates is 
\begin{equation}
K=O\!\left(1+\frac a{\epsilon^2}
                +\frac{a^{3/2}}{n\epsilon^3}\right).
\end{equation}
\section{The proposed methods for finite-sum problems}
\label{sec:finite-main}
For the finite-sum problem~\eqref{eq:fs-problem}, an occasional full pass
through the data can replace the stochastic correction used by DVR-Sign
and DVR-Q. We compute exact gradients
at checkpoints, followed by compressed component-gradient differences.
We use a deterministic sign for the $\ell_1$ criterion and employ the unbiased compressor $Q$ for the $\ell_2$ criterion.

Fix an integer refresh
period $q=m$. 
At times $t=1+kq$, every
worker computes and sends its local full gradient, and the
server resets $z_t=\nabla f(x_t)$. At other times, each worker independently draws
a uniform component $i_t^j\in\{1,\ldots,m\}$ and sends a compressed
paired difference, using the same component at both iterates. The server uses
\begin{equation}
y_t^j=\nabla f_{j,i_t^j}(x_t)-\nabla f_{j,i_t^j}(x_{t-1}),\qquad
z_t=z_{t-1}+\frac1n\sum_{j=1}^n Q(y_t^j).
\label{eq:fs-recursion}
\end{equation}

\begin{algorithm}[!htbp]
\caption{DVR-Sign-FS and DVR-Q-FS}
\label{alg:dvr-fs}
\begin{algorithmic}[1]
\Require $x_1$, $K$, $\eta$,
refresh period $q$, and compressor $Q$.
\For{$t=1,\ldots,K$}
  \If{$t=1+kq$ for an integer $k\ge0$}
    \State Each worker computes and sends $\nabla f_j(x_t)=m^{-1}\sum_{i=1}^m\nabla f_{j,i}(x_t)$.
    \State Server sets $z_t=n^{-1}\sum_{j=1}^n\nabla f_j(x_t)$.
  \Else
    \State Every worker $j$ draws an independent uniform $i_t^j\in\{1,\ldots,m\}$.
    \State Every worker forms $y_t^j$ in~\eqref{eq:fs-recursion} and sends $Q(y_t^j)$.
    \State Server sets $z_t=z_{t-1}+n^{-1}\sum_{j=1}^n Q(y_t^j)$.
  \EndIf
  \If{DVR-Sign-FS}
    \State Server broadcasts $s_t=\Sign(z_t)$.
  \Else
    \State Server broadcasts $s_t=Q(z_t)$.
  \EndIf
  \State Set $x_{t+1}=x_t-\eta s_t$.
\EndFor
\State Draw an independent uniform $\tau\in\{1,\ldots,K\}$ and \Return $x_\tau$.
\end{algorithmic}
\end{algorithm}
The server retains $z_t$ itself for the next estimator
update.
Component smoothness gives
$\norm{y_t^j}\le L\norm{x_t-x_{t-1}}$, which controls the increment
without requiring bounded gradients or bounded heterogeneity. The whole method is described in Algorithm~\ref{alg:dvr-fs}. 

\subsection{Deterministic sign updates for the \texorpdfstring{$\ell_1$}{l1} guarantee}
\label{sec:fs-l1}
The first version uses $s_t=\Sign(z_t)$. A refresh eliminates the estimation error and only
the $q-1$ intervening updates can contribute to its variance. A refresh costs $M$ component-gradient evaluations, whereas an ordinary
update costs $2n$. Choosing $q=m$ makes the amortized refresh cost
$M/q=n$, of the same order as the ordinary update cost.
\begin{theorem}
\label{thm:fs-master}
Under Assumptions~\ref{ass:compressor} and~\ref{ass:finite-sum}, let
$a=1+\omega$. For DVR-Sign-FS, define
\begin{equation}
J_1=\frac d2+2d\sqrt{\frac{a(q-1)}n}.
\label{eq:fs-geometry}
\end{equation}
For an independent uniform $\tau\in\{1,\ldots,K\}$,
\begin{align}
\E\norm{z_t-\nabla f(x_t)}^2
&\le\frac{aL^2\eta^2d(q-1)}n,
\label{eq:fs-error}\\
\E\normone{\nabla f(x_\tau)}
&\le\frac{\Delta_0}{\eta K}+L\eta J_1.
\label{eq:fs-master}
\end{align}
\end{theorem}

\begin{corollary}
\label{cor:fs-complexity-l1}
Under the assumptions of Theorem~\ref{thm:fs-master}, set $q=m$. We ensure
$\E\normone{\nabla f(x_\tau)}\le\epsilon$, with
\begin{align}
\eta=\frac{\epsilon}{2LJ_1},\quad
K=&O\!\left(1+\frac{L\Delta_0d}{\epsilon^2}
                         \left[1+\sqrt{\frac{a(m-1)}n}\right]\right).\\
N_{\rm grad}
=&O\!\left(M+\frac{L\Delta_0d}{\epsilon^2}[n+\sqrt{aM}]\right).
\label{eq:fs-gradient-complexity-l1}
\end{align}
\end{corollary}
\textbf{Remark:}
For $n\leq O(a m)$, the total oracle order is $N_{\rm grad}=O\!\left(M+\frac{d\sqrt{aM}}{\epsilon^2}\right)$.

\subsection{Unbiased compressed updates for the \texorpdfstring{$\ell_2$}{l2} guarantee}
\label{sec:fs-l2}
The second version keeps the same refreshes and estimator, but uses
$s_t=Q(z_t)$, where $Q$ satisfies
Assumption~\ref{ass:compressor} with $a=1+\omega$.
Unlike a sign step of fixed length, this update has conditional second
moment at most $a\eta^2\norm{z_t}^2$. Unbiasedness also makes the
expected descent depend on $\langle\nabla f(x_t),z_t\rangle$.
These two facts allow the tracking error to be absorbed into the descent
inequality and yield a squared Euclidean stationarity guarantee.

\begin{theorem}
\label{thm:fs-master-l2}
Under Assumptions~\ref{ass:compressor} and~\ref{ass:finite-sum}, for $a=1+\omega$, define
\begin{equation}
J_Q=a\left(1+\sqrt{\frac{q-1}{n}}\right),
\qquad \eta=\frac1{2LJ_Q}.
\label{eq:fs-geometry-l2}
\end{equation}
Then an independent uniform
$\tau\in\{1,\ldots,K\}$ satisfies
\begin{equation}
\E\norm{\nabla f(x_\tau)}^2
\le\frac{4L\Delta_0J_Q}{K},\qquad
\E\norm{\nabla f(x_\tau)}
\le2\sqrt{\frac{L\Delta_0J_Q}{K}}.
\label{eq:fs-master-l2}
\end{equation}
\end{theorem}

\begin{corollary}
\label{cor:fs-complexity-l2}
Under the assumptions of Theorem~\ref{thm:fs-master-l2}, set $q=m$. We ensure
$\E\norm{\nabla f(x_\tau)}\le\epsilon$, with
\begin{align}
K&=O\!\left(1+\frac{aL\Delta_0}{\epsilon^2}
                         \left[1+\sqrt{\frac{m-1}{n}}\right]\right),
\label{eq:fs-update-complexity-l2}\\
N_{\rm grad}
&=O\!\left(M+\frac{aL\Delta_0}{\epsilon^2}
                                      [n+\sqrt M]\right).
\label{eq:fs-gradient-complexity-l2}
\end{align}
\end{corollary}
\textbf{Remark:} For $n\leq O(\sqrt m)$, this total oracle order $N_{\rm grad}=O\!\left(M+\frac{a\sqrt{M}}{\epsilon^2}\right)$ matches the
centralized Euclidean benchmark of SPIDER and PAGE
\citep{fang2018spider,li2021page}. 

\vspace{0.3in}

\clearpage
\bibliography{references}
\bibliographystyle{plainnat}

\clearpage
\appendix
\numberwithin{theorem}{section}
\numberwithin{assumption}{section}
\numberwithin{equation}{section}
\section{Proof of Proposition~\ref{prop:majority-floor}}
\label{app:majority}
We first prove Proposition~\ref{prop:majority-floor} for the SSVR-MV Option~1. 
Use the three one-dimensional functions
\begin{equation}
f_1(x)=f_2(x)=\tfrac12\log\cosh x+\tfrac14x,
\qquad f_3(x)=\tfrac12\log\cosh x-\tfrac12x.
\end{equation}
Since $\cosh x\ge e^{|x|}/2$, we have
$\log\cosh x\ge |x|-\log2$. It follows that
\[
f_1(x)=f_2(x)\ge\tfrac12|x|+\tfrac14x-\tfrac12\log2
\ge-\tfrac12\log2,
\quad
f_3(x)\ge\tfrac12(|x|-x)-\tfrac12\log2\ge-\tfrac12\log2.
\]
Thus each local objective is lower bounded. Their derivatives are
\[
f_1'(x)=f_2'(x)=\tfrac12\tanh x+\tfrac14,\qquad
f_3'(x)=\tfrac12\tanh x-\tfrac12.
\]
Because $|\tanh x|\le1$, all three derivatives have absolute value at
most $G=1$. Moreover, $f_j''(x)=1/(2\cosh^2x)\in(0,1/2]$,
so local objectives are $1/2$-smooth. Their average and its derivative are
\[
f(x)=\tfrac12\log\cosh x,\qquad h(x):=f'(x)=\tfrac12\tanh x.
\]
In particular, we observe that $x_1=0$ is a global minimizer and $|h(x)|\le1/2$.
Take the deterministic oracle $g_j(x;\xi)=f_j'(x)$ whose variance is zero.

Exact initialization gives $v_1^j=f_j'(x_1)$. If
$v_{t-1}^j=f_j'(x_{t-1})$, the local recursion gives
\[
v_t^j=f_j'(x_t)+(1-\beta)
       [f_j'(x_{t-1})-f_j'(x_{t-1})]=f_j'(x_t).
\]
Induction therefore proves exact local estimates at every iteration
for every $\beta\in(0,1]$. The signs $q_t^j=S_4(f_j'(x_t))$ are well defined because
$|f_j'(x_t)|\le1<4$. 

Next, to keep the following scalar calculation short, write $u=1/16$ and
$z=h(x)/4$. The three sign means at $x_t=x$ are
$z+u,z+u,z-2u$. Expanding the product,
\[
(z+u)^2(z-2u)
=(z^2+2uz+u^2)(z-2u)=z^3-3u^2z-2u^3.
\]
Consequently the expected vote, denoted by $M(x)$, is
\begin{equation}
M(x):=\E[s_t\mid x_t=x]
=\frac{3z-(z+u)^2(z-2u)}2
=\frac{3(1+u^2)z-z^3+2u^3}{2}.
\label{eq:majority-conditional}
\end{equation}
At the minimizer, $h(0)=z=0$ and $M(0)=u^3>0$.
The following derivative bound controls how much this bias can change
when the true gradient is small:
\begin{equation}
\frac{dM}{dh}=\frac38(1+u^2-z^2),\qquad
0<\frac{dM}{dh}\le\frac38(1+u^2)=\frac{771}{2048}.
\label{eq:majority-slope}
\end{equation}
Indeed $|z|=|h|/4\le1/8$, so $1+u^2-z^2>0$.
The mean value theorem now gives
\begin{equation}
|M(x)-u^3|\le\frac{771}{2048}|h(x)|.
\label{eq:majority-lipschitz}
\end{equation}
Next, we control the iterate without assuming it is bounded.
Since $\tanh1>1/2$, $x\ge1$ implies that $z\ge u$, and $x\le-1$
implies $z\le-u$. The polynomial in \eqref{eq:majority-conditional}
is increasing in $z$ on $[-1/8,1/8]$.
Its values at $u$ and $-u$ are, respectively,
$(3u+4u^3)/2>0$ and $-3u/2<0$. Thus $xM(x)\ge0$ for $|x|\ge1$.
For $|x|<1$, the fact that $M(x)$ is the expectation of a sign gives
$|M(x)|\le1$ and hence $xM(x)\ge-1$. Combining the two regions yields
$xM(x)\ge-1$ for every $x$.

The update $x_{t+1}=x_t-\eta s_t$ and $s_t^2=1$ give
\[
\E[x_{t+1}^2\mid x_t]
=x_t^2-2\eta x_tM(x_t)+\eta^2.
\]
Take expectations and sum over $t=1,\ldots,K$. Since $x_1=0$,
\begin{equation}
\E x_{K+1}^2
=\sum_{t=1}^K\bigl(-2\eta\E[x_tM(x_t)]+\eta^2\bigr)
\le2\eta K+\eta^2K.
\label{eq:majority-iterate}
\end{equation}
All moments are finite, since for a fixed horizon,
$|x_t|\le(t-1)\eta$ on every sample path.

Taking expectations gives
$\E x_{t+1}-\E x_t=-\eta\E M(x_t)$.
After summing and using $x_1=0$,
\[
\frac1K\sum_{t=1}^K\E M(x_t)=-\frac{\E x_{K+1}}{\eta K}.
\]
Cauchy--Schwarz and \eqref{eq:majority-iterate} imply
\begin{equation}
\left|\frac1K\sum_{t=1}^K\E M(x_t)\right|
\le\frac{\sqrt{\E x_{K+1}^2}}{\eta K}
\le\sqrt{\frac2{\eta K}+\frac1K}.
\label{eq:majority-average}
\end{equation}
This argument controls the average vote even though individual iterates
need not converge.

Finally, the triangle
inequality and \eqref{eq:majority-average} give
\begin{align*}
u^3
&\le\left|\frac1K\sum_{t=1}^K\E M(x_t)\right|
  +\left|u^3-\frac1K\sum_{t=1}^K\E M(x_t)\right|\\
&\le\sqrt{\frac2{\eta K}+\frac1K}
  +\frac1K\sum_{t=1}^K\E|u^3-M(x_t)|\\
&\le\sqrt{\frac2{\eta K}+\frac1K}
  +\frac{771}{2048}\frac1K\sum_{t=1}^K\E|h(x_t)|.
\end{align*}
By independent uniform output selection, the final average is
$\E|f'(x_\tau)|$. Since $u^3=1/4096$, rearranging gives the finite-horizon bound
\begin{equation}
\E|f'(x_\tau)|
\ge\frac{2048}{771}
\left[\frac1{4096}-\sqrt{\frac2{\eta K}+\frac1K}\right].
\end{equation}
For $\eta=\Theta(K^{-1/2})$, we have $\eta K\to\infty$,
so the square-root term tends to zero. Taking the lower limit and using
$2048/(771\cdot4096)=1/1542$ proves
\eqref{eq:majority-floor-main}.

\section{Proofs for the stochastic results}
\label{app:general}
We first prove a tracking estimate that retains the actual model
movement.
Throughout, $a=1+\omega$, $c=a/n$, $K$ counts model updates, and $B_0$
is the initialization batch per worker.

\subsection{Moments and server tracking}
\label{app:proofs}
For $t\ge2$, let $\mathcal F_{t-1}$ contain all randomness generated
before the current worker samples and compression calls. Write
$\E_t[\cdot]=\E[\cdot\mid\mathcal F_{t-1}]$. First, unbiased compression controls second moments. Conditional on an
input $v$, expansion of the squared norm gives
\begin{align}
\E[\norm{Q(v)}^2\mid v]
=\norm v^2+2\langle v,\E[Q(v)-v\mid v]\rangle
  +\E[\norm{Q(v)-v}^2\mid v]\nonumber\le a\norm v^2.
\end{align}
The middle term is zero by unbiasedness. Applying
$\norm{u+v}^2\le2\norm u^2+2\norm v^2$ to the decomposition
in~\eqref{eq:update}, we obtain
\begin{align}
\E_t\norm{h_t^j}^2
&\le2\beta^2\E_t\norm{g_j(x_t;\xi_t^j)}^2
 +2(1-\beta)^2\E_t
   \norm{g_j(x_t;\xi_t^j)-g_j(x_{t-1};\xi_t^j)}^2\nonumber\\
&\le2\beta^2H^2+2L^2\norm{x_t-x_{t-1}}^2.
\label{eq:conditional-update}
\end{align}
The last inequality uses mean-squared smoothness and $(1-\beta)^2\le1$.

Next, we control the initial error. Write
$Y_{j,b}=Q(g_j(x_1;\xi_{1,b}^j))$ for the message from sample $b$ at
worker $j$. Repeated conditioning gives
\[
\E Y_{j,b}=\nabla f_j(x_1),\qquad
\E\norm{Y_{j,b}}^2\le aH^2.
\]
Its centered second moment therefore satisfies
\[
\E\norm{Y_{j,b}-\nabla f_j(x_1)}^2
=\E\norm{Y_{j,b}}^2-\norm{\nabla f_j(x_1)}^2\le aH^2.
\]
Distinct initialization messages use independent samples and compression
calls. For $(j,b)\ne(k,r)$, their centered cross term is consequently
\[
\E\langle Y_{j,b}-\nabla f_j(x_1),
          Y_{k,r}-\nabla f_k(x_1)\rangle=0.
\]
Expanding the squared norm of the average gives
\begin{equation}
    E_1
=\frac1{n^2B_0^2}\sum_{j=1}^n\sum_{b=1}^{B_0}
  \E\norm{Y_{j,b}-\nabla f_j(x_1)}^2\nonumber\le\frac{nB_0aH^2}{n^2B_0^2}=\frac{cH^2}{B_0}.
\end{equation}
This explains why each initialization sample is compressed separately. We now derive the tracking recursion. For $t\ge2$, define the centered
decoded error
\[
\delta_t^j
=Q(h_t^j)-[\nabla f_j(x_t)-(1-\beta)\nabla f_j(x_{t-1})].
\]
Unbiasedness of the oracle and of the compressor implies
\[
\E_t Q(h_t^j)=\E_t h_t^j
=\nabla f_j(x_t)-(1-\beta)\nabla f_j(x_{t-1}),
\qquad \E_t\delta_t^j=0.
\]
Subtracting $\nabla f(x_t)$ from~\eqref{eq:server-recursion} gives the
exact identity
\begin{equation}
e_t=(1-\beta)e_{t-1}+\frac1n\sum_{j=1}^n\delta_t^j.
\label{eq:error-identity}
\end{equation}
The conditional squared norm expands as
\begin{align}
\E_t\norm{e_t}^2
&=(1-\beta)^2\norm{e_{t-1}}^2
 +\frac{2(1-\beta)}n\sum_j
       \langle e_{t-1},\E_t\delta_t^j\rangle\nonumber\\
&\quad+\frac1{n^2}\sum_j\E_t\norm{\delta_t^j}^2
 +\frac2{n^2}\sum_{j<k}\E_t\langle\delta_t^j,\delta_t^k\rangle.
\end{align}
The second term is zero because $e_{t-1}$ is fixed under conditioning.
The last term is zero because the current worker errors are conditionally
independent and centered.
For each remaining variance, centering and~\eqref{eq:conditional-update}
give
\begin{align}
\E_t\norm{\delta_t^j}^2
&=\E_t\norm{Q(h_t^j)}^2
  -\norm{\nabla f_j(x_t)-(1-\beta)\nabla f_j(x_{t-1})}^2\nonumber\\
&\le a\E_t\norm{h_t^j}^2
\le2a\beta^2H^2+2aL^2\norm{x_t-x_{t-1}}^2.
\end{align}
Taking full expectations and using $(1-\beta)^2\le1-\beta$ yields
\begin{equation}
E_t\le(1-\beta)E_{t-1}+2c\beta^2H^2
                  +2cL^2\E\norm{x_t-x_{t-1}}^2.
\label{eq:tracking-movement}
\end{equation}

For completeness, summing this recursion from $t=2$ to $K$ gives
\begin{align}
\beta\sum_{t=1}^K E_t
&\le E_1-(1-\beta)E_K+2c\beta^2H^2(K-1)
       +2cL^2\sum_{t=1}^{K-1}\E\norm{x_{t+1}-x_t}^2\nonumber\\
&\le E_1+2c\beta^2H^2K
       +2cL^2\sum_{t=1}^{K-1}\E\norm{x_{t+1}-x_t}^2.
\end{align}
After division by $\beta K$, this proves
\begin{equation}
\overline E_K\le\frac{cH^2}{B_0\beta K}+2c\beta H^2
 +\frac{2cL^2}{\beta K}
   \sum_{t=1}^{K-1}\E\norm{x_{t+1}-x_t}^2.
\label{eq:average-movement}
\end{equation}
For DVR-Sign,
$\norm{x_{t+1}-x_t}^2=\eta^2d$. Substituting this identity
into~\eqref{eq:tracking-movement} and~\eqref{eq:average-movement}
proves~\eqref{eq:tracking} and~\eqref{eq:average-error}. The DVR-Q
specialization is given separately in Appendix~\ref{app:l2}.

\subsection{The \texorpdfstring{$\ell_1$}{l1} guarantee}
\label{app:stochastic-l1}
Let $g$ be a true gradient and $z$ an arbitrary estimate. If
$\Sign(g_k)\ne\Sign(z_k)$ and $g_k\ne0$, then $z_k$ is on the opposite
side of zero or is zero, and therefore $|g_k|\le|g_k-z_k|$.
Coordinates with $g_k=0$ contribute zero. It follows that
\begin{align}
\normone g-\langle g,\Sign(z)\rangle
&=2\sum_{k=1}^d |g_k|\ind{\Sign(g_k)\ne\Sign(z_k)}\nonumber\\
&\le2\sum_{k=1}^d |g_k-z_k|
=2\normone{g-z}\le2\sqrt d\norm{g-z}.
\label{eq:unified-alignment}
\end{align}
Next, we sum the objective decreases.
Mean-squared smoothness and Jensen's inequality give
\[
\norm{\nabla f_j(x)-\nabla f_j(y)}
=\norm{\E[g_j(x;\xi)-g_j(y;\xi)]}
\le\bigl(\E\norm{g_j(x;\xi)-g_j(y;\xi)}^2\bigr)^{1/2}
\le L\norm{x-y}.
\]
Averaging over workers shows that $f$ is also $L$-smooth. Applying its
smoothness inequality to the deterministic sign update gives
\begin{align}
f(x_{t+1})
&\le f(x_t)-\eta\langle\nabla f(x_t),\Sign(z_t)\rangle
      +\frac{L\eta^2}{2}\norm{\Sign(z_t)}^2\nonumber\\
&\le f(x_t)-\eta\normone{\nabla f(x_t)}
      +2\eta\sqrt d\norm{e_t}+\frac{L\eta^2d}{2}.
\end{align}
Taking expectations and summing from $t=1$ to $K$, we have
\[
\eta\sum_{t=1}^K\E\normone{\nabla f(x_t)}
\le\del+2\eta\sqrt d\sum_{t=1}^K\E\norm{e_t}
             +\frac{KL\eta^2d}{2}.
\]
For the error sum, Jensen's inequality followed by Cauchy--Schwarz gives
\[
\frac1K\sum_{t=1}^K\E\norm{e_t}
\le\frac1K\sum_{t=1}^K\sqrt{E_t}
\le\frac1K\sqrt{K\sum_{t=1}^K E_t}
=\sqrt{\overline E_K}.
\]
Dividing the previous descent inequality by $\eta K$ proves
\begin{equation}
\frac1K\sum_{t=1}^K\E\normone{\nabla f(x_t)}
\le\frac{\del}{\eta K}+\frac{L\eta d}{2}
+2\sqrt{d\,\overline E_K}.
\label{eq:descent-master}
\end{equation}
Because $\tau$ is uniform on $\{1,\ldots,K\}$ and independent of the run,
its expected gradient norm equals the average on the left. This proves
Theorem~\ref{thm:master}. For Corollary~\ref{cor:rate}, choose
\begin{equation}
u_1=\frac1{\sqrt K+c^{1/3}K^{2/3}},\qquad
\beta=u_1,\quad \eta=\frac{Hu_1}{L\sqrt d},\quad
B_0=\left\lceil\frac1{u_1^2K}\right\rceil.
\label{eq:rate-parameters}
\end{equation}
These are admissible because $0<u_1\le K^{-1/2}\le1$ and $B_0\ge1$.
Inserting the parameters into~\eqref{eq:average-error} controls each
contribution separately:
\[
\frac{cH^2}{B_0\beta K}\le cH^2u_1,\qquad
2cH^2\beta=2cH^2u_1,\qquad
\frac{2cL^2\eta^2d}{\beta}=2cH^2u_1.
\]
Thus $\overline E_K\le5cH^2u_1$, and~\eqref{eq:descent-master} becomes
\[
\E\normone{\nabla f(x_\tau)}
\le\sqrt d\left[
\frac{L\del}{Hu_1K}+\frac{Hu_1}{2}
+2\sqrt5H\sqrt{cu_1}\right].
\]
The definition of $u_1$ gives
\[
\frac1{u_1K}=K^{-1/2}+c^{1/3}K^{-1/3},\qquad
u_1\le K^{-1/2},\qquad
u_1\le c^{-1/3}K^{-2/3}.
\]
In particular, the third inequality implies
$\sqrt{cu_1}\le c^{1/3}K^{-1/3}$. Substitution proves the explicit bound
\begin{equation}
\E\normone{\nabla f(x_\tau)}
\le\sqrt d\left[
\left(\frac{L\del}{H}+\frac H2\right)K^{-1/2}
+\left(\frac{L\del}{H}+2\sqrt5H\right)
  \left(\frac cK\right)^{1/3}\right].
\label{eq:rate}
\end{equation}
Keeping $L,H,\del$ fixed and recalling $c=(1+\omega)/n$ gives
\eqref{eq:rate-l1}.

Finally, we specify initialization and a target accuracy.
Squaring the denominator of $u_1$ yields the exact batch size
\begin{equation}
B_0=\left\lceil(1+c^{1/3}K^{1/6})^2\right\rceil
=\Theta(1+c^{2/3}K^{1/3}).
\label{eq:batch}
\end{equation}
To verify the order including the ceiling, use
$1+b^2\le(1+b)^2\le2(1+b^2)$ for $b\ge0$, and
$y\le\lceil y\rceil\le y+1\le2y$ for $y\ge1$.
For any $\epsilon>0$, the following explicit integer horizon suffices:
\begin{equation}
K=\max\left\{
1,\
\left\lceil\frac{4d(L\del/H+H/2)^2}{\epsilon^2}\right\rceil,\
\left\lceil\frac{8cd^{3/2}(L\del/H+2\sqrt5H)^3}
                   {\epsilon^3}\right\rceil
\right\}.
\label{eq:stochastic-accuracy}
\end{equation}
The second entry makes the first term of~\eqref{eq:rate} at most
$\epsilon/2$: square that desired inequality and solve for $K$.
The third entry does the same for the second term by cubing it.
Their sum is therefore at most $\epsilon$. With fixed $L,H,\del$,
\begin{equation}
K=O\!\left(1+\frac d{\epsilon^2}
              +\frac{cd^{3/2}}{\epsilon^3}\right).
\label{eq:stochastic-horizon}
\end{equation}

\subsection{The \texorpdfstring{$\ell_2$}{l2} guarantee}
\label{app:l2}
Let $\mathcal G_t$ denote the history after the server
has formed $z_t$ and before drawing that downlink. Thus
\begin{equation}
\E[Q(z_t)\mid\mathcal G_t]=z_t,\qquad
\E[\norm{Q(z_t)}^2\mid\mathcal G_t]\le a\norm{z_t}^2.
\label{eq:q-downlink-moments}
\end{equation}
For clarity, write $Z_t=\E\norm{z_t}^2$ only within this proof. The
actual movement satisfies
\[
\E\norm{x_{t+1}-x_t}^2
=\eta^2\E\norm{Q(z_t)}^2\le a\eta^2Z_t.
\]
Inserting it into the general tracking bounds gives
\begin{align}
E_t&\le(1-\beta)E_{t-1}+2c\beta^2H^2
                         +2caL^2\eta^2Z_{t-1}\quad(t\ge2),
\label{eq:tracking-q}\\
\overline E_K&\le\frac{cH^2}{B_0\beta K}+2c\beta H^2
                  +\frac{2caL^2\eta^2}{\beta K}
                       \sum_{t=1}^{K-1}Z_t.
\label{eq:average-error-q}
\end{align}
Unlike the sign update, this movement becomes small when the tracker
becomes small. We must retain its dependence on $Z_t$ in the descent
calculation.

By $L$-smoothness and~\eqref{eq:q-downlink-moments},
\begin{align}
\E[f(x_{t+1})\mid\mathcal G_t]
&\le f(x_t)-\eta\langle\nabla f(x_t),z_t\rangle
                    +\frac{La\eta^2}{2}\norm{z_t}^2.
\end{align}
The exact inner-product identity
\[
2\langle\nabla f(x_t),z_t\rangle
=\norm{\nabla f(x_t)}^2+\norm{z_t}^2-\norm{e_t}^2
\]
then gives, after taking full expectations,
\begin{align}
\E f(x_{t+1})
&\le\E f(x_t)-\frac\eta2\E\norm{\nabla f(x_t)}^2
          -\frac\eta2(1-La\eta)Z_t+\frac\eta2E_t.
\end{align}
Summing from $t=1$ to $K$, using $\E f(x_{K+1})\ge f_*$, and dividing
by $\eta K/2$ yields
\begin{equation}
\frac1K\sum_{t=1}^K\E\norm{\nabla f(x_t)}^2
+\frac{1-La\eta}{K}\sum_{t=1}^K Z_t
\le\frac{2\del}{\eta K}+\overline E_K.
\label{eq:descent-master-l2}
\end{equation}
Substitute~\eqref{eq:average-error-q} and use
$\sum_{t=1}^{K-1}Z_t\le\sum_{t=1}^KZ_t$. Moving this last contribution
to the left gives
\begin{align}
&\frac1K\sum_{t=1}^K\E\norm{\nabla f(x_t)}^2
+\left(1-La\eta-\frac{2caL^2\eta^2}{\beta}\right)
  \frac1K\sum_{t=1}^KZ_t\nonumber\\
&\hspace{18mm}\le\frac{2\del}{\eta K}
                +\frac{cH^2}{B_0\beta K}+2c\beta H^2.
\end{align}
The independent uniform output $\tau$ turns the first average into
$\E\norm{\nabla f(x_\tau)}^2$. If $La\eta\le1/2$ and
$\beta\ge4caL^2\eta^2$, both subtracted terms in the tracker
coefficient are at most $1/2$, so the coefficient is nonnegative.
This proves Theorem~\ref{thm:master-l2}. The cancellation explains why
we did not replace $Z_t$ by a gradient-magnitude bound earlier.

We next prove Corollary~\ref{cor:rate-l2}. Recall $c=a/n$ and choose
\begin{equation}
r=(K/n^2)^{1/3},\qquad
\eta=\frac1{2La(1+r)},\quad
\beta=\frac1{n(1+r)^2},\quad B_0=\lceil1+r\rceil.
\label{eq:rate-parameters-l2}
\end{equation}
These choices satisfy $0<\beta\le1$ for every $n,K\ge1$. Direct
substitution gives
\[
La\eta=\frac1{2(1+r)},\qquad
4caL^2\eta^2=\frac1{n(1+r)^2}=\beta.
\]
The tracker coefficient in~\eqref{eq:master-l2} is consequently
\[
1-\frac1{2(1+r)}-\frac12=\frac r{2(1+r)}\ge0.
\]
We can drop that term and bound the three remaining terms separately:
\begin{align}
\frac{2\del}{\eta K}
&=\frac{4La\del(1+r)}K,\nonumber\\
\frac{cH^2}{B_0\beta K}
&=\frac{aH^2(1+r)^2}{B_0K}
 \le\frac{aH^2(1+r)}K,\nonumber\\
2c\beta H^2
&=\frac{2aH^2}{n^2(1+r)^2}
 =\frac{2aH^2r^3}{K(1+r)^2}
 \le\frac{2aH^2r}K.
\end{align}
The second bound uses $B_0\ge1+r$. The last uses
$r^3=K/n^2$ and $r^2\le(1+r)^2$. Adding the bounds and using
$r/K=(nK)^{-2/3}$ proves
\begin{equation}
\E\norm{\nabla f(x_\tau)}^2
\le a\left[\frac{4L\del+H^2}{K}
       +\frac{4L\del+3H^2}{(nK)^{2/3}}\right].
\label{eq:rate-constants-l2}
\end{equation}
Jensen's inequality and $\sqrt{u+v}\le\sqrt u+\sqrt v$ for $u,v\ge0$
give
\begin{align}
\E\norm{\nabla f(x_\tau)}
&\le\sqrt{\E\norm{\nabla f(x_\tau)}^2}\nonumber\\
&\le\sqrt{4L\del+H^2}\sqrt{\frac aK}
 +\sqrt{4L\del+3H^2}\frac{\sqrt a}{(nK)^{1/3}}.
\end{align}
This proves~\eqref{eq:rate-l2}. To obtain $\E\norm{\nabla f(x_\tau)}^2\le\epsilon^2$, and hence
$\E\norm{\nabla f(x_\tau)}\le\epsilon$, it suffices to choose
\begin{equation}
\begin{split}
K=\max\Biggl\{1,
\left\lceil\frac{2a(4L\del+H^2)}{\epsilon^2}\right\rceil,\left\lceil\frac{[2a(4L\del+3H^2)]^{3/2}}
                         {n\epsilon^3}\right\rceil\Biggr\}.
\end{split}
\label{eq:stochastic-accuracy-l2}
\end{equation}
The second entry makes the first term in~\eqref{eq:rate-constants-l2}
at most $\epsilon^2/2$. Raising the desired bound on its second term
to the power $3/2$ gives the third entry. Their sum is at most
$\epsilon^2$. For fixed $L,H,\del$, the resulting update and per-worker
sample complexities are
\begin{equation}
K=O\!\left(1+\frac a{\epsilon^2}
                +\frac{a^{3/2}}{n\epsilon^3}\right).
\label{eq:stochastic-horizon-l2}
\end{equation}
\section{Proofs for finite sums}
\label{app:finite-sum}

\subsection{Exact computation}
For $K$ updates and an integer
refresh period $q\ge1$, the times $1,1+q,1+2q,\ldots$ give
\[
r=1+\left\lfloor\frac{K-1}{q}\right\rfloor
 =\left\lceil\frac Kq\right\rceil.
\]
Each refresh evaluates every component once for $M=nm$ evaluations.
At each of the $K-r$ ordinary updates, every worker evaluates one
component at two points. Therefore
\begin{equation}
N_{\rm grad}=Mr+2n(K-r),\qquad
N_{{\rm grad},j}=mr+2(K-r).
\label{eq:fs-exact-gradient-counts}
\end{equation}
For $q=m$, the inequality $r\le1+K/m$ gives
\begin{equation}
N_{\rm grad}\le M\left(1+\frac Km\right)+2nK
=M+3nK.
\label{eq:fs-cost-bound}
\end{equation}
\subsection{Tracking error between exact refreshes}
\begin{lemma}
\label{lemma:fs-tracking}
For the estimator~\eqref{eq:fs-recursion}, the error is zero at each
refresh $t_0=1+kq$. For $t_0\le t<t_0+q$, we have
\begin{equation}
\E\norm{z_t-\nabla f(x_t)}^2
\le\frac{aL^2}{n}\sum_{s=t_0+1}^{t}
                             \E\norm{x_s-x_{s-1}}^2.
\label{eq:fs-local-error}
\end{equation}
Deterministic sign updates therefore give the bound
$aL^2\eta^2d(q-1)/n$. For updates $x_{t+1}=x_t-\eta Q(z_t)$,
with a fresh common downlink satisfying the same compressor assumption,
\begin{equation}
\frac1K\sum_{t=1}^K\E\norm{z_t-\nabla f(x_t)}^2
\le\frac{a^2L^2\eta^2(q-1)}{n}
                         \frac1K\sum_{t=1}^K\E\norm{z_t}^2.
\label{eq:fs-q-average-error}
\end{equation}
\end{lemma}
\begin{proof}\leavevmode
At an ordinary iteration, condition on the history through $x_t$ and
before the current component indices and compression calls. Denote this
conditional expectation by $\E_t$; the points $x_t,x_{t-1}$ and the
previous estimate $z_{t-1}$ are fixed under this conditioning.
Write
\begin{align*}
X_j&=Q\bigl(\nabla f_{j,i_t^j}(x_t)
                   -\nabla f_{j,i_t^j}(x_{t-1})\bigr),\\
\mu_j&=\E_tX_j=\nabla f_j(x_t)-\nabla f_j(x_{t-1}).
\end{align*}
Uniform component sampling and unbiased compression give the expression
for $\mu_j$. The same component is evaluated at both points. Thus
\begin{align}
\E_t\norm{X_j}^2
\le \frac a m\sum_{i=1}^m
       \norm{\nabla f_{j,i}(x_t)-\nabla f_{j,i}(x_{t-1})}^2
\nonumber\le aL^2\norm{x_t-x_{t-1}}^2.
\label{eq:fs-candidate-moment}
\end{align}
The first inequality uses~\eqref{eq:compressor-second}; the second is
component smoothness. Since $X_j-\mu_j$ is centered,
\[
\E_t\norm{X_j-\mu_j}^2
=\E_t\norm{X_j}^2-\norm{\mu_j}^2
\le aL^2\norm{x_t-x_{t-1}}^2.
\]
Note that the component indices and compressor calls are conditionally independent
across all $n$ workers. For $j\ne k$, this independence and centering imply
\[
\E_t\langle X_j-\mu_j,X_k-\mu_k\rangle
=\langle\E_t(X_j-\mu_j),\E_t(X_k-\mu_k)\rangle=0.
\]
Consequently,
\begin{align}
\E_t\left\|\frac1n\sum_{j=1}^n(X_j-\mu_j)\right\|^2
=\frac1{n^2}\sum_{j=1}^n\E_t\norm{X_j-\mu_j}^2
\nonumber \le\frac{aL^2}{n}\norm{x_t-x_{t-1}}^2.
\label{eq:fs-worker-variance}
\end{align}
The averaging is over all workers and
$n^{-1}\sum_j\mu_j=\nabla f(x_t)-\nabla f(x_{t-1})$ exactly.

Next, let $e_t=z_t-\nabla f(x_t)$. The estimator recursion gives
\[
e_t=e_{t-1}+\frac1n\sum_{j=1}^n(X_j-\mu_j).
\]
The new sum has zero conditional mean, and $e_{t-1}$ is fixed under
$\E_t$. Expanding the square yields
\begin{align*}
\E_t\norm{e_t}^2
&=\norm{e_{t-1}}^2
 +2\left\langle e_{t-1},\frac1n\sum_{j=1}^n
                                     \E_t(X_j-\mu_j)\right\rangle
 +\E_t\left\|\frac1n\sum_{j=1}^n(X_j-\mu_j)\right\|^2\\
&\le\norm{e_{t-1}}^2+\frac{aL^2}{n}\norm{x_t-x_{t-1}}^2.
\end{align*}
At a refresh, $z_{t_0}=\nabla f(x_{t_0})$, so $e_{t_0}=0$ on every
sample path. Taking total expectations and applying the recursion
successively at $t_0+1,\ldots,t$ gives
\[
\E\norm{e_t}^2
\le\E\norm{e_{t_0}}^2
       +\frac{aL^2}{n}\sum_{s=t_0+1}^t
                                     \E\norm{x_s-x_{s-1}}^2.
\]
This proves~\eqref{eq:fs-local-error}. For DVR-Sign-FS,
each summand is $\eta^2d$, and there are at most $q-1$ summands.

For DVR-Q-FS, condition on the history $\mathcal G_s$ after all uplinks
in iteration $s$ and before the server's new compression call. Then
\[
\E[\norm{x_{s+1}-x_s}^2\mid\mathcal G_s]
=\eta^2\E[\norm{Q(z_s)}^2\mid\mathcal G_s]
\le a\eta^2\norm{z_s}^2.
\]
Substituting its total expectation in~\eqref{eq:fs-local-error} yields
\[
\E\norm{e_t}^2
\le\frac{a^2L^2\eta^2}{n}
                   \sum_{s=t_0}^{t-1}\E\norm{z_s}^2.
\]
Let $t_1=\min\{t_0+q-1,K\}$ be the last iteration in this interval
between refreshes. Reversing the order of the finite sums gives
\begin{align*}
\sum_{t=t_0}^{t_1}\E\norm{e_t}^2
&\le\frac{a^2L^2\eta^2}{n}
           \sum_{s=t_0}^{t_1-1}(t_1-s)\E\norm{z_s}^2\\
&\le\frac{a^2L^2\eta^2(q-1)}{n}
           \sum_{s=t_0}^{t_1}\E\norm{z_s}^2.
\end{align*}
These intervals are disjoint. Summing over them and dividing by $K$
proves~\eqref{eq:fs-q-average-error}.
\end{proof}

\subsection{Descent for the \texorpdfstring{$\ell_1$}{l1} criterion}
\begin{proof}[Proof of Theorem~\ref{thm:fs-master}]
The deterministic sign $s_t=\Sign(z_t)$ has squared norm $d$.
Thus the deterministic sign case of Lemma~\ref{lemma:fs-tracking}
proves~\eqref{eq:fs-error}. To translate this into stationarity,
smoothness gives
\[
f(x_{t+1})\le f(x_t)-\eta\langle\nabla f(x_t),\Sign(z_t)\rangle
                                  +\frac{L\eta^2d}{2}.
\]
By~\eqref{eq:alignment},
\[
\langle\nabla f(x_t),\Sign(z_t)\rangle
\ge\normone{\nabla f(x_t)}
                    -2\sqrt d\norm{z_t-\nabla f(x_t)}.
\]
Substituting and rearranging before taking expectations gives
\[
\eta\E\normone{\nabla f(x_t)}
\le\E f(x_t)-\E f(x_{t+1})
 +2\eta\sqrt d\E\norm{z_t-\nabla f(x_t)}+\frac{L\eta^2d}{2}.
\]
Sum this inequality over $t=1,\ldots,K$. The objective terms telescope
to $f(x_1)-\E f(x_{K+1})\le f(x_1)-f_*=\del\le\Delta_0$.
After division by $\eta K$, we obtain
\begin{align*}
\frac1K\sum_{t=1}^K\E\normone{\nabla f(x_t)}
&\le\frac{\Delta_0}{\eta K}+\frac{L\eta d}{2}
       +\frac{2\sqrt d}{K}\sum_{t=1}^K
                              \E\norm{z_t-\nabla f(x_t)}\\
&\le\frac{\Delta_0}{\eta K}+\frac{L\eta d}{2}
       +2\sqrt d\sqrt{\frac{aL^2\eta^2d(q-1)}n}\\
&=\frac{\Delta_0}{\eta K}
       +L\eta\left[\frac d2+2d\sqrt{\frac{a(q-1)}n}\right].
\end{align*}
In the second line, $\E\norm{z_t-\nabla f(x_t)}
\le\sqrt{\E\norm{z_t-\nabla f(x_t)}^2}$ and
Lemma~\ref{lemma:fs-tracking} bound each summand. Uniformity of $\tau$ implies that the average on the left equals
$\E\normone{\nabla f(x_\tau)}$. This proves~\eqref{eq:fs-master}.

Finally, $J_1\ge d/2>0$, so the positive stepsize is well defined and
$L\eta J_1=\epsilon/2$. If $\Delta_0>0$, the choice of $K$ gives
\[
\frac{\Delta_0}{\eta K}
=\frac{2L\Delta_0J_1}{\epsilon K}\le\frac\epsilon2.
\]
\end{proof}
\paragraph{Complexity}
Set $q=m$.
The target-accuracy choice gives, for every $\epsilon>0$,
\[
K \le 1+\frac{4L\Delta_0J_1}{\epsilon^2}
=1+\frac{4L\Delta_0d}{\epsilon^2}
                   \left[\frac12+2\sqrt{\frac{a(m-1)}n}\right].
\]
For the total oracle count, multiply the coefficient by $n$:
\begin{equation}
nJ_1=\frac{dn}{2}+2d\sqrt{an(m-1)}
\le\frac{dn}{2}+2d\sqrt{aM}.
\label{eq:fs-tuning-bounds}
\end{equation}
The inequality uses $n(m-1)\le nm=M$.
Substituting the iteration bound into~\eqref{eq:fs-cost-bound} now gives
\begin{align*}
N_{\rm grad}
&\le M+3n+\frac{12L\Delta_0nJ_1}{\epsilon^2}\\
&\le4M+\frac{12L\Delta_0d}{\epsilon^2}
                     \left[\frac n2+2\sqrt{aM}\right].
\end{align*}
Here $n\le M$ absorbs the extra update arising from the ceiling.

\subsection{Descent for the \texorpdfstring{$\ell_2$}{l2} criterion}
\begin{proof}[Proof of Theorem~\ref{thm:fs-master-l2}]\leavevmode
Condition on the post-uplink history $\mathcal G_t$, which includes
$x_t$ and $z_t$ and precedes the server's new call to $Q$.
Unbiasedness and the second-moment bound give
\[
\E[Q(z_t)\mid\mathcal G_t]=z_t,\qquad
\E[\norm{Q(z_t)}^2\mid\mathcal G_t]\le a\norm{z_t}^2.
\]
Since $x_{t+1}=x_t-\eta Q(z_t)$, smoothness implies
\[
\E[f(x_{t+1})\mid\mathcal G_t]
\le f(x_t)-\eta\langle\nabla f(x_t),z_t\rangle
                                      +\frac{La\eta^2}{2}\norm{z_t}^2.
\]
The inner product has the exact decomposition
\[
2\langle\nabla f(x_t),z_t\rangle
=\norm{\nabla f(x_t)}^2+\norm{z_t}^2
                            -\norm{z_t-\nabla f(x_t)}^2.
\]
Substitute this identity, take total expectations, and rearrange:
\begin{align*}
\frac\eta2\E\norm{\nabla f(x_t)}^2
+\frac\eta2(1-La\eta)\E\norm{z_t}^2
\le \E f(x_t)-\E f(x_{t+1})
                +\frac\eta2\E\norm{z_t-\nabla f(x_t)}^2.
\end{align*}
Sum the last inequality over $t=1,\ldots,K$ and divide by $\eta K/2$.
The objective values telescope, and $f(x_1)-\E f(x_{K+1})\le\Delta_0$.
Applying~\eqref{eq:fs-q-average-error} to the right side gives
\begin{align}
\frac1K\sum_{t=1}^K\E\norm{\nabla f(x_t)}^2
&+\left[1-La\eta-\frac{a^2L^2\eta^2(q-1)}n\right]
                \frac1K\sum_{t=1}^K\E\norm{z_t}^2
\nonumber\\
&\le\frac{2\Delta_0}{\eta K}.
\label{eq:fs-q-descent}
\end{align}
Write $s=\sqrt{(q-1)/n}\ge0$.
The choice in~\eqref{eq:fs-geometry-l2} gives
\[
La\eta+\frac{a^2L^2\eta^2(q-1)}n
=\frac1{2(1+s)}+\frac{s^2}{4(1+s)^2}
=\frac14+\frac1{4(1+s)^2}\le\frac12.
\]
Thus the bracket in~\eqref{eq:fs-q-descent} is at least $1/2$,
and we may discard its nonnegative contribution. Since
$2\Delta_0/(\eta K)=4L\Delta_0J_Q/K$, we obtain the squared-norm
bound in~\eqref{eq:fs-master-l2}. The independent uniform choice of
$\tau$ identifies its expectation with the averaged left side.
Cauchy--Schwarz then gives
\[
\E\norm{\nabla f(x_\tau)}
\le\sqrt{\E\norm{\nabla f(x_\tau)}^2}
\le2\sqrt{\frac{L\Delta_0J_Q}{K}}.
\]
\end{proof}
\paragraph{Complexity}
For $q=m$, we have
$J_Q=a(1+\sqrt{(m-1)/n})$. Hence
\[
K\le1+\frac{4L\Delta_0J_Q}{\epsilon^2}
=1+\frac{4aL\Delta_0}{\epsilon^2}
                    \left[1+\sqrt{\frac{m-1}{n}}\right],
\]
which proves~\eqref{eq:fs-update-complexity-l2}.
To count gradients, use
\[
nJ_Q=a\left[n+\sqrt{n(m-1)}\right]\le a(n+\sqrt M).
\]
Then~\eqref{eq:fs-cost-bound} yields
\begin{align*}
N_{\rm grad}
&\le M+3n+\frac{12L\Delta_0nJ_Q}{\epsilon^2}\\
&\le4M+\frac{12aL\Delta_0}{\epsilon^2}(n+\sqrt M).
\end{align*}
\end{document}